\documentclass{article} 
\usepackage{times,epsfig,wrapfig,amsmath,amsfonts,amssymb,amsthm,bm,color,enumitem,algorithm,algpseudocode}

\newcommand{\argmax}[1]{\underset{#1}{\mathrm{argmax}} \:}
\newcommand{\inner}[1]{\left\langle #1 \right\rangle}

\newtheorem{theorem}{Theorem}

\newcommand{\norm}[1]{\left\lVert{#1}\right\rVert}
\newcommand{\abs}[1]{\left\lvert{#1}\right\rvert}

\newcommand{\sign}{\operatorname{sign}}

\newcommand{\R}{\mathbb{R}}

\newcommand{\defeq}{\triangleq}
\usepackage{SIGACTNews}
\usepackage{times,color}
\usepackage{dsfont}
\usepackage{enumitem}

\newcommand{\cX}{\mathcal{X}}
\newcommand{\pmo}{\{\pm 1\}}
\newcommand{\unitmatrix}{\mathds{1}}
\newcommand{\LinLin}{\textsc{Lin:Lin} }
\newcommand{\LinV}{\textsc{Lin:V} }
\graphicspath{{Figures/}{../figures/}}

\title{The Power of Asymmetry in Binary Hashing}

\author{
Behnam Neyshabur\qquad Payman Yadollahpour\qquad Yury Makarychev\\
Toyota Technological Institute at Chicago\\
\texttt{[btavakoli,pyadolla,yury]@ttic.edu} \\
\and
Ruslan Salakhutdinov \\
Departments of Statistics and Computer Science \\
University of Toronto \\
\texttt{rsalakhu@cs.toronto.edu} \\
\and
Nathan Srebro \\
Toyota Technological Institute at Chicago\\
and Technion, Haifa, Israel\\
\texttt{nati@ttic.edu} \\
}
\date{}
\begin{document}

\maketitle
\begin{abstract}
  When approximating binary similarity using the hamming distance between
  short binary hashes, we show that even if the similarity is
  symmetric, we can have shorter and more accurate hashes by using two
  distinct code maps. I.e.~by approximating the similarity between $x$
  and $x'$ as the hamming distance between $f(x)$ and $g(x')$, for two
  distinct binary codes $f,g$, rather than as the hamming distance
  between $f(x)$ and $f(x')$.
\end{abstract}
\section{Introduction}

Encoding high-dimensional objects using short binary hashes can be
useful for fast approximate similarity computations and nearest
neighbor searches.  Calculating the hamming distance between two short
binary strings is an extremely cheap computational operation, and the
communication cost of sending such hash strings for lookup on a server
(e.g.~sending hashes of all features or patches in an image taken on a
mobile device) is low.  Furthermore, it is also possible to quickly
look up nearby hash strings in populated hash tables.  Indeed, it only
takes a fraction of a second to retrieve a shortlist of similar items
from a corpus containing billions of data points, which is important
in image, video, audio, and document retrieval
tasks~\cite{SalakhutdinovH09,NorouziNIPS12,RaginskyNIPS09,TorralbaCVPR08}.
Moreover, compact binary codes are remarkably storage efficient, and
allow one to store massive datasets in memory.  It is therefore
desirable to find short binary hashes that correspond well to some
target notion of similarity.  Pioneering work on Locality Sensitive
Hashing used random linear thresholds for obtaining bits of the hash
\cite{DatarIIM04}.  Later work suggested learning hash functions
attuned to the distribution of the data
\cite{WeissTF08,SalakhutdinovH09,Kulis2009,LiuICML11,GongTPAMI12}. More recent work focuses
on learning hash functions so as to optimize agreement with the target
similarity measure on specific datasets
\cite{WangICML2010,NorouziICML11,NorouziNIPS12,LiuCVPR12} .  It is
important to obtain accurate and {\em short} hashes---the
computational and communication costs scale linearly with the length
of the hash, and more importantly, the memory cost of the hash table
can scale exponentially with the length.

In all the above-mentioned approaches, similarity $S(x,x')$ between
two objects is approximated by the hamming distance between the
outputs of the same hash function, i.e.~between $f(x)$ and $f(x')$,
for some $f\in\pmo^k$.  The emphasis here is that the same hash
function is applied to both $x$ and $x'$ (in methods like LSH multiple
hashes might be used to boost accuracy, but the comparison is still
between outputs of the same function).

The only exception we are aware of is where a single mapping of
objects to fractional vectors $\tilde{f}(x) \in [-1,1]^k$ is used, its
thresholding $f(x)=\sign \tilde{f}(x) \in \{\pm 1\}^k$ is used in the
database, and similarity between $x$ and $x'$ is approximated using
$\inner{f(x),\tilde{f}(x')}$.  This has become known as ``asymmetric
hashing'' \cite{DongSIGIR08,GordoCVPR11}, but even with such a-symmetry, both mappings are
based on the same fractional mapping $\tilde{f}(\cdot)$.  That is, the
asymmetry is in that one side of the comparison gets thresholded while
the other is fractional, but not in the actual mapping.

In this paper, we propose using two {\em distinct} mappings $f(x),g(x)
\in \{\pm 1 \}^k$ and approximating the similarity $S(x,x')$ by the
hamming distance between $f(x)$ and $g(x')$.  We refer to such hashing
schemes as ``asymmetric''.  Our main result is that even if the target
similarity function is symmetric and ``well behaved'' (e.g., even if
it is based on Euclidean distances between objects), using asymmetric
binary hashes can be much more powerful, and allow better
approximation of the target similarity with shorter code lengths.  In
particular, we show extreme examples of collections of points in
Euclidean space, where the neighborhood similarity $S(x,x')$ can be
realized using an asymmetric binary hash (based on a pair of distinct
functions) of length $O(r)$ bits, but where a symmetric hash (based on
a single function) would require at least $\Omega(2^r)$ bits.
Although actual data is not as extreme, our experimental results
on real data sets demonstrate significant benefits from using
asymmetric binary hashes. 

Asymmetric hashes can be used in almost all places where symmetric
hashes are typically used, usually without any additional storage or
computational cost.  Consider the typical application of storing hash
vectors for all objects in a database, and then calculating
similarities to queries by computing the hash of the query and its
hamming distance to the stored database hashes.  Using an asymmetric
hash means using different hash functions for the database
and for the query.  This neither increases the size of the database
representation, nor the computational or communication cost of
populating the database or performing a query, as the exact same
operations are required.  In fact, when hashing the entire database,
asymmetric hashes provide even more opportunity for improvement.  We
argue that using two different hash functions to encode database
objects and queries allows for much more flexibility in choosing the
database hash.  Unlike the query hash, which has to be stored
compactly and efficiently evaluated on queries as they appear, if the
database is fixed, an arbitrary mapping of database objects
to bit strings may be used.  We demonstrate
that this can indeed increase similarity accuracy while
reducing the bit length required.

\section{Minimum Code Lengths and the Power of Asymmetry}

Let $S: \cX \times \cX \rightarrow \{ \pm 1 \}$ be a binary similarity
function over a set of objects $\cX$, where we can interpret $S(x,x')$
to mean that $x$ and $x'$ are ``similar'' or ``dissimilar'', or to
indicate whether they are ``neighbors''.  A symmetric binary coding of
$\cX$ is a mapping $f:\cX \rightarrow \{ \pm 1 \}^k$, where $k$ is the
bit-length of the code.  We are interested in constructing codes such
that the hamming distance between $f(x)$ and $f(x')$ corresponds to
the similarity $S(x,x')$.  That is, for some threshold $\theta \in
\R$, $S(x,x') \approx \sign( \inner{f(x),f(x')}-\theta )$.  Although
discussing the hamming distance, it is more convenient for us to work
with the inner product $\inner{u,v}$, which is equivalent to the
hamming distance $d_h(u,v)$ since $\inner{u,v} = (k - 2d_h(u,v))$ for
$u,v \in \{ \pm 1\}^k$.

In this section, we will consider the problem of capturing a given
similarity using an arbitrary binary code.  That is, we are given the
entire similarity mapping $S$, e.g.~as a matrix $S \in \{\pm 1\}^{n
  \times n}$ over a finite domain $\cX=\{x_1,\ldots,x_n\}$ of $n$
objects, with $S_{ij}=S(x_i,x_j)$.  We ask for an encoding
$u_i=f(x_i)\in\{\pm 1\}^k$ of each object $x_i \in \cX$, and a
threshold $\theta$, such that $S_{ij} = \sign(
\inner{u_i,u_j}-\theta)$, or at least such that equality holds for as
many pairs $(i,j)$ as possible.  It is important to emphasize that the
goal here is purely to approximate the given matrix $S$ using a short
binary code---there is no out-of-sample generalization (yet).

We now ask: Can allowing an asymmetric coding enable approximating a
symmetric similarity matrix $S$ with a shorter code length?

Denoting $U \in \{\pm 1\}^{n \times k}$ for the matrix whose columns
contain the codewords $u_i$, the minimal binary code length that allows
exactly representing $S$ is then given by the following matrix
factorization problem:
\begin{equation}
  \label{eq:ks}
k_s(S) = \;\; \underset{k,U,\theta}{\min} \; k\quad 
\textrm{s.t}\quad
\begin{aligned}[t]
&U\in \{\pm 1\}^{k\times n}\quad\quad\theta\in\R\\
& Y \defeq U^{\top}U -\theta \unitmatrix_n\\
&\forall_{ij} \; S_{ij} Y_{ij}  >0
\end{aligned}
\end{equation}
where $\unitmatrix_n$ is an $n\times n$ matrix of ones.

We begin demonstrating the power of asymmetry by considering an {\em
  asymmetric} variant of the above problem.  That is, even if $S$ is
symmetric, we allow associating with each object $x_i$ two distinct
binary codewords, $u_i \in \pmo^k$ and $v_i\in\pmo^k$ (we can think of
this as having two arbitrary mappings $u_i=f(x_i)$ and
$v_i=g(x_i)$), such that $S_{ij} = \sign( \inner{u_i,v_j}-\theta)$.
The minimal asymmetric binary code length is then given by:
\begin{equation}
  \label{eq:ka}
  k_a(S) = \;\;\underset{k,U,V,\theta}{\min} \; k\quad 
\textrm{s.t}\quad
\begin{aligned}[t]
&U,V\in \{\pm 1\}^{k\times n}\quad\quad\theta\in\R\\
& Y\defeq U^{\top}V - \theta\unitmatrix_n\\
&\forall_{ij} \; S_{ij}  Y_{ij}  >0\\
\end{aligned}
\end{equation}

Writing the binary coding problems as matrix factorization problems is
useful for understanding the power we can get by asymmetry: even if
$S$ is symmetric, and even if we seek a symmetric $Y$, insisting on
writing $Y$ as a square of a binary matrix might be a tough
constraint.  This is captured in the following Theorem, which
establishes that there could be an exponential gap between the minimal
asymmetry binary code length and the minimal symmetric code length,
even if the matrix $S$ is symmetric and very well behaved:

\begin{theorem}\label{thm:gap}
  For any $r$, there exists a set of $n=2^r$ points in Euclidean
  space, with similarity matrix $S_{ij}=
  \begin{cases}
    1 & \textrm{if $\norm{x_i-x_j}\leq 1$} \\
    -1 & \textrm{if $\norm{x_i-x_j}>1$}
  \end{cases}$, such that $k_a(S) \leq 2 r$ but $k_s(S) \geq 2^r/2$
\end{theorem}
\begin{proof}
Let $I_1 = \{1,\dots, n/2\}$ and $I_2 = \{n/2+1,\dots, n\}$.
Consider the matrix $G$ defined by $G_{ii}=1/2$, $G_{ij} = -1/(2n)$ if $i,j\in I_1$ or  
$i,j\in I_2$, and $G_{ij} = 1/(2n)$ otherwise. Matrix $G$ is diagonally dominant.
By the Gershgorin circle theorem, $G$ is positive definite. Therefore, there exist 
vectors $x_1, \dots, x_n$ such that $\inner{x_i, x_j} = G_{ij}$ (for every $i$ and $j$).
Define 
$$S_{ij}=
  \begin{cases}
    1 & \textrm{if $\norm{x_i-x_j}\leq 1$} \\
    -1 & \textrm{if $\norm{x_i-x_j}>1$}
  \end{cases}.$$
Note that if $i=j$ then $S_{ij}=1$; if $i\neq j$ and $(i,j)\in I_1\times I_1 \cup I_2\times I_2$ then $\norm{x_i-x_j}^2 = G_{ii} + G_{jj} - 2G_{ij} = 1 +1/n > 1$ and therefore $S_{ij}=-1$. Finally, if $i\neq j$ and $(i,j)\in I_1\times I_2 \cup I_2\times I_1$ then 
$\norm{x_i-x_j}^2 = G_{ii} + G_{jj} - 2G_{ij} = 1 + 1/n < 1$ and therefore $S_{ij}=1$. We show that 
$k_a(S) \leq 2 r$. Let $B$ be an $r\times n$ matrix whose column vectors are the vertices of the cube $\{\pm1\}^r$ (in any order);
let  $C$  be  an $r\times n$ matrix  defined by $C_{ij} = 1$ if $j\in I_1$ and $C_{ij} = -1$ if $j\in I_2$.
Let 
$U=
\begin{bmatrix} B \\ C\end{bmatrix}$ and $V=
\begin{bmatrix}  B \\ -C\end{bmatrix}$.
For $Y = U^\top V - \theta \unitmatrix_n$ where threshold $\theta=-1$ , we have that $Y_{ij} \geq 1$ if $S_{ij}=1$ and $Y_{ij} \leq -1$ if $S_{ij}=-1$.
Therefore, $k_a(S) \leq 2 r$. 

Now we show that $k_s= k_s(S) \geq n/2$. Consider $Y$, $U$ and $\theta$ as in (\ref{eq:ks}). 
Let $Y' = (U^{\top} U)$. Note that $Y'_{ij} \in [-k_s,k_s]$ and thus $\theta \in [-k_s+1,k_s-1]$.
Let $q = [{1,\dots,1},{-1,\dots,-1}]^{\top}$ ($n/2$ ones followed by $n/2$ minus ones).
We have,
\begin{align*}
0 \leq q^\top Y'q &= \sum_{i=1}^n Y'_{ii} + \sum_{i,j:S_{ij}=-1} Y'_{ij} - \sum_{i,j: S_{ij}=1, i\neq j} Y'_{ij}\\
&\leq \sum_{i=1}^n k_s + \sum_{i,j:S_{ij}=-1} (\theta-1)
- \sum_{i,j:S_{ij}=1,i\neq j} (\theta + 1)\\
&=nk_s +(0.5n^2-n)(\theta-1) -0.5n^2(\theta+1)\\
&= nk_s-n^2-n(\theta-1) \\
&\leq 2nk_s -n^2.
\end{align*}
We conclude that $k_s \geq n/2$.
\end{proof}

The construction of Theorem \ref{thm:gap} shows that there {\em
  exists} data sets for which an asymmetric binary hash might be
much shorter then a symmetric hash.  This is an important observation
as it demonstrates that asymmetric hashes could be much more powerful,
and should prompt us to consider them instead of symmetric hashes.
The precise construction of Theorem \ref{thm:gap} is of course rather
extreme (in fact, the most extreme construction possible) and we would
not expect actual data sets to have this exact structure, but we will
show later significant gaps also on real data sets.  

\section{Approximate Binary Codes}

As we turn to real data sets, we also need to depart from seeking a
binary coding that {\em exactly} captures the similarity matrix.
Rather, we are usually satisfied with merely approximating $S$, and
for any fixed code length $k$ seek the (symmetric or asymmetric)
$k$-bit code that ``best captures'' the similarity matrix $S$.  This
is captured by the following optimization problem:
\begin{equation}
  \label{eq:aproxUV}
\min_{U,V,\theta}\;\;  
L(Y;S) \triangleq  
\beta\!\!\!\!\!\!\sum_{i,j:S_{ij}=1}\!\!\!\!\!\! \ell(Y_{ij}) +
(1-\beta)\!\!\!\!\!\!\!\! \sum_{i,j:S_{ij}=-1}\!\!\!\!\!\!\!\!
\ell(-Y_{ij})\quad
\begin{aligned}[t]
\textrm{s.t.}\;\;
&U,V\in \{\pm 1\}^{k\times n}\quad\theta\in\R\\
& Y\defeq U^{\top}V - \theta\unitmatrix_n
\end{aligned}
\end{equation}
where $\ell(z)=\mathbf{1}_{z\leq 0}$ is the zero-one-error and
$\beta$ is a parameter that allows us to weight
positive and negative errors differently.  Such weighting can
compensate for $S_{ij}$ being imbalanced (typically many more pairs of
points are non-similar rather then similar), and allows us to obtain
different balances between precision and recall.

The optimization problem \eqref{eq:aproxUV} is a discrete,
discontinuous and highly non-convex problem.  In our experiments, we
replace the zero-one loss $\ell(\cdot)$ with a continuous loss and
perform local search by greedily updating single bits so as to improve
this objective.  Although the resulting objective (let alone the
discrete optimization problem) is still not convex even if $\ell(z)$
is convex, we found it beneficial to use a loss function that is not
flat on $z<0$, so as to encourage moving towards the correct sign.  In
our experiments, we used the square root of the logistic loss,
$\ell(z) = \log^{1/2}(1+e^{-z})$.

Before moving on to out-of-sample generalizations, we briefly report
on the number of bits needed empirically to find good approximations
of actual similarity matrices with symmetric and asymmetric codes.  We
experimented with several data sets, attempting to fit them with both
symmetric and asymmetric codes, and then calculating average precision
by varying the threshold $\theta$ (while keeping $U$ and $V$ fixed).
Results for two similarity matrices, one based on Euclidean distances
between points uniformly distributed in a hypoercube, and the other
based on semantic similarity between images, are shown in Figure \ref{fig:fitmat}.

\begin{figure}[t!]
\vspace{0.2in}
\hbox{ \centering
\setlength{\epsfxsize}{3.1in}
\epsfbox{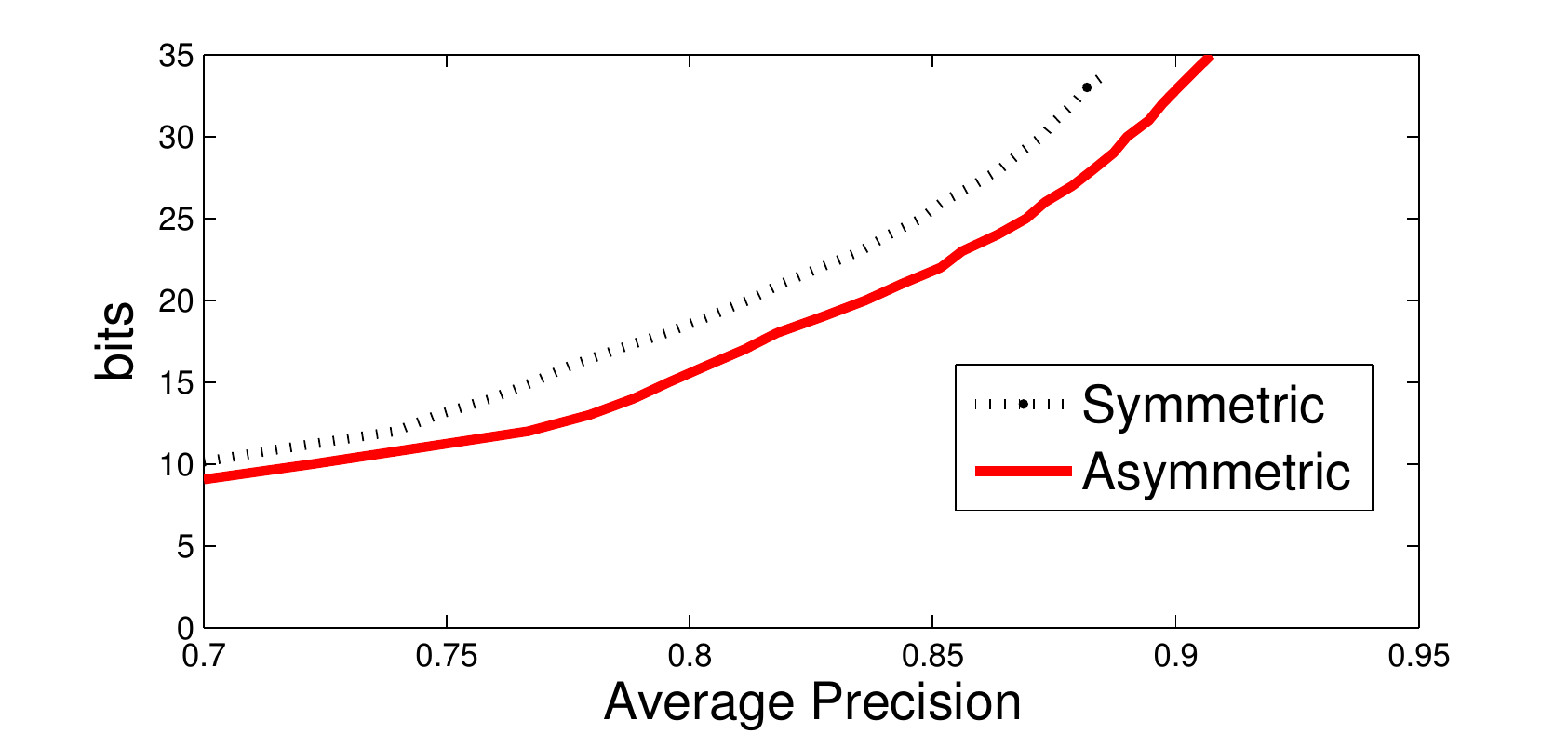}
\setlength{\epsfxsize}{3.1in}
\epsfbox{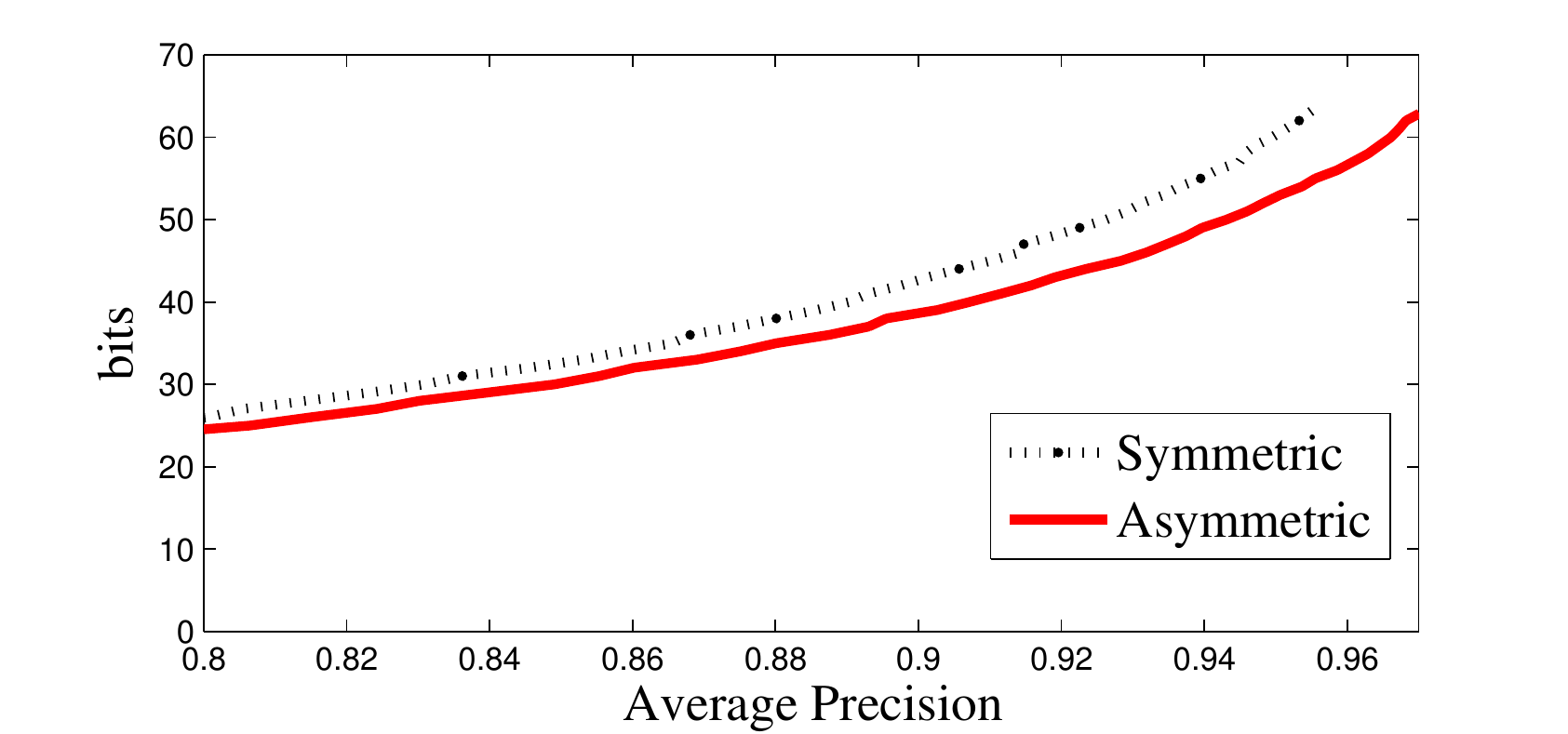}
\begin{picture}(0,0)(0,0)
\put(-370, +104){ 10-D Uniform}
\put(-135, +104){ LabelMe }
\end{picture}
}
\small\caption{\small Number of bits required for approximating two similarity
  matrices (as a function of average precision).  Left: uniform data
  in the 10-dimensional hypercube, similarity represents a thresholded
  Euclidean distance, set such that 30\% of the similarities are
  positive.  Right: Semantic similarity of a subset of LabelMe images,
  thresholded such that 5\% of the similarities are positive. 
}
\label{fig:fitmat}
\end{figure}
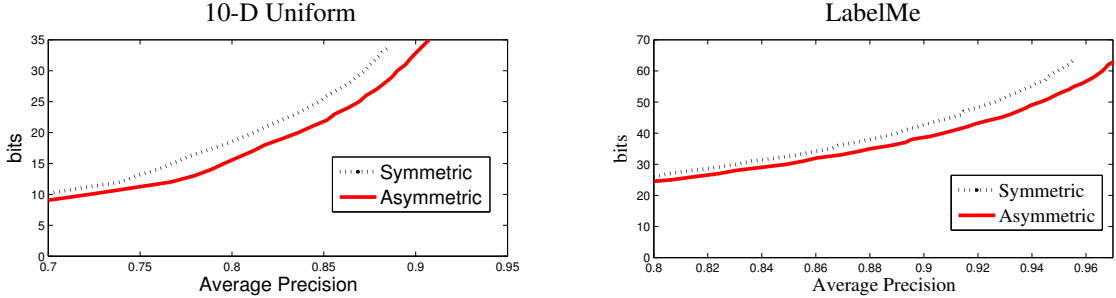



\section{Out of Sample Generalization: Learning a Mapping}

So far we focused on learning binary codes over a fixed set of objects by 
associating an arbitrary code word with each object and completely
ignoring the input representation of the objects $x_i$.
We discussed only how well binary hashing can {\em approximate} the
similarity, but did not consider {\em generalizing} to additional new
objects.  However, in most applications, we would like to be able to
have such an out-of-sample generalization. That is, we would like to
learn a mapping $f:\cX\rightarrow \pmo^k$ over an infinite domain
$\cX$ using only a finite training set of objects, and then apply the
mapping to obtain binary codes $f(x)$ for future objects to be
encountered, such that $S(x,x') \approx \sign(
\inner{f(x),f(x')}-\theta )$.  Thus, the mapping
$f:\cX\rightarrow\pmo^k$ is usually limited to some constrained
parametric class, both so we could represent and evaluate it
efficiently on new objects, and to ensure good generalization.  For
example, when $\cX=\R^d$, we can consider linear threshold mappings
$f_W(x) = \sign(Wx)$, where $W\in\R^{k\times d}$ and $\sign(\cdot)$
operates elementwise, as in Minimal Loss Hashing
\cite{NorouziICML11}.  Or, we could also consider more complex classes, such as multilayer networks
\cite{SalakhutdinovH09,NorouziNIPS12}.

We already saw that asymmetric binary codes can allow for better
approximations using shorter codes, so it is natural to seek
asymmetric codes here as well.  That is, instead of learning a single
parametric map $f(x)$ we can learn a pair of maps
$f:\cX\rightarrow\pmo^k$ and $g:\cX\rightarrow\pmo^k$, both
constrained to some parametric class, and a threshold $\theta$, such
that $S(x,x') \approx \sign( \inner{f(x),g(x')}-\theta )$.  This has
the potential of allowing for better approximating the similarity, and
thus better overall accuracy with shorter codes (despite possibly
slightly harder generalization due to the increase in the number of
parameters).

In fact, in a typical application where a database of objects is
hashed for similarity search over future queries, asymmetry allows us
to go even further.  Consider the following setup: We are given $n$
objects $x_1,\ldots,x_n \in \cX$ from some infinite domain $\cX$ and
the similarities $S(x_i,x_j)$ between these objects. Our goal is to
hash these objects using short binary codes which would allow us to
quickly compute approximate similarities between these objects (the
``database'') and future objects $x$ (the ``query'').  That is, we
would like to generate and store compact binary codes for objects in a
database.  Then, given a new query object, we would like to
efficiently compute a compact binary code for a given query and
retrieve similar items in the database very fast by finding binary
codes in the database that are within small hamming distance from the
query binary code.  Recall that it is important to ensure that the bit
length of the hashes are small, as short codes allow for very fast
hamming distance calculations and low communication costs if the codes
need to be sent remotely. More importantly, if we would like to
store the database in a hash table allowing immediate lookup, the size
of the hash table is exponential in the code length.

The symmetric binary hashing approach (e.g.~\cite{NorouziICML11}),
would be to find a single parametric mapping $f:\cX\rightarrow\pmo^k$
such that $S(x,x_i) \approx \sign( \inner{f(x),f(x_i)}-\theta )$ for
future queries $x$ and database objects $x_i$, calculate $f(x_i)$ for
all database objects $x_i$, and store these hashes (perhaps in a hash
table allowing for fast retrieval of codes within a short hamming
distance).  The asymmetric approach described above would be to find
two parametric mappings $f:\cX\rightarrow\pmo^k$ and
$g:\cX\rightarrow\pmo^k$ such that $S(x,x_i) \approx \sign(
\inner{f(x),g(x_i)}-\theta )$, and then calculate and store $g(x_i)$.

But if the database is fixed, we can go further.  There is actually no
need for $g(\cdot)$ to be in a constrained parametric class, as we do
not need to generalize $g(\cdot)$ to future objects, nor do we have to
efficiently calculate it on-the-fly nor communicate $g(x)$ to the
database.  Hence, we can consider allowing the database hash function
$g(\cdot)$ to be an arbitrary mapping.  That is, we aim to find a
simple parametric mapping $f:\cX\rightarrow\pmo^k$ and $n$ {\em
  arbitrary codewords} $v_1,\ldots,v_n \in\pmo^k$ for each
$x_1,\ldots,x_n$ in the database, such that $S(x,x_i) \approx
\sign(\inner{f(x),v_i}-\theta )$ for future queries $x$ and for the
objects $x_i,\ldots,x_n$ in the database.  This form of asymmetry can
allow us for greater approximation power, and thus better accuracy
with shorter codes, at no additional computational or storage cost.

In Section \ref{sec:emp} we evaluate empirically both of the above asymmetric
strategies and demonstrate their benefits. But before doing so, in the
next Section, we discuss a local-search approach for finding the
mappings $f,g$, or the mapping $f$ and the codes $v_1,\ldots,v_n$.

\section{Optimization}
\label{Sec:Optimization} 
We focus on $x\in\cX\subset\R^d$ and linear threshold hash maps of the
form $f(x)=\sign(Wx)$, where $W\in\R^{k\times d}$.  Given training
points $x_1,\ldots,x_n$, we consider the two models discussed above:
\begin{description}[topsep=0pt,parsep=0pt,partopsep=0pt]
\item[\LinLin] We learn two linear threshold functions $f(x)=\sign(W_q
  x)$ and $g(x)=\sign(W_d x)$.  I.e.~we need to find the parameters
  $W_q,W_d\in\R^{k\times d}$.
\item[\LinV] We learn a single linear threshold function $f(x)=\sign(W_q
  x)$ and $n$ codewords $v_1,\ldots,v_n\in\R^k$.  I.e.~we need to find
  $W_q\in\R^{k\times d}$, as well as $V\in\R^{k\times n}$ (where $v_i$
  are the columns of $V$).
\end{description}
In either case we denote $u_i=f(x_i)$, and in \LinLin also
$v_i=g(x_i)$, and learn by attempting to minimizing the objective in
\eqref{eq:aproxUV}, where $\ell(\cdot)$ is again a continuous loss
function such as the square root of the logistic.  That is, we learn
by optimizing the problem \eqref{eq:aproxUV} with the additional
constraint $U=\sign(W_q X)$, and possibly also $V=\sign(W_d X)$ (for
$\LinLin$), where $X=[x_1 \ldots x_n]\in\R^{d\times n}$.

We optimize these problems by alternatively updating rows
of $W_q$ and either rows of $W_d$ (for \LinLin) or of $V$ (for \LinV).
To understand these updates, let us first return to \eqref{eq:aproxUV}
(with unconstrained $U,V$), and consider updating a row
$u^{(t)}\in\R^n$ of $U$.  Denote
$$Y^{(t)}=U^{\top}V-\theta\unitmatrix_n-{u^{(t)}}^{\top}v^{(t)},$$ the prediction
matrix with component $t$ subtracted away.  It is easy to verify that
we can write:
\begin{equation}
  \label{eq:usingM}
  L(U^{\top}V-\theta\unitmatrix_n;S) = C - u^{(t)} M {v^{(t)}}^{\top}
\end{equation}
where $C=\frac{1}{2}(L(Y^{(t)}+\unitmatrix_n;S)+L(Y^{(t)}-\unitmatrix_n;S))$ does not depend on
$u^{(t)}$ and $v^{(t)}$, and $M\in\R^{n\times n}$ also does not depend
on $u^{(t)},v^{(t)}$ and is given by:
$$
M_{ij}= \frac{\beta_{ij}}{2} \left(\ell(S_{ij}(Y_{ij}^{(t)}-1)) -
  \ell(S_{ij}(Y_{ij}^{(t)}+1))\right),$$ with $\beta_{ij}=\beta$ or
$\beta_{ij}=(1-\beta)$ depending on $S_{ij}$.  This implies that we
can optimize over the entire row $u^{(t)}$ concurrently by maximizing
$u^{(t)} M {v^{(t)}}^{\top}$, and so the optimum (conditioned on
$\theta$, $V$ and all other rows of $U$) is given by
\begin{equation}
  \label{eq:utopt}
 u^{(t)} = \sign(M v^{(t)}). 
\end{equation}
Symmetrically, we can optimize over the row $v^{(t)}$ conditioned on
$\theta$, $U$ and the rest of $V$, or in the case of \LinV, conditioned
on $\theta$, $W_q$ and the rest of $V$.

Similarly, optimizing over a row $w^{(t)}$ of $W_q$ amount to
optimizing:
\begin{equation}
  \label{eq:optWt}
  \arg\max_{w^{(t)}\in\R^d}\; \sign(w^{(t)}X) M {v^{(t)}}^{\top} = \arg\max_{w^{(t)}\in\R^d}\;\sum_i \inner{M_i,v^{(t)}} \sign(\inner{w^{(t)},x_i}).
\end{equation}
This is a weighted zero-one-loss binary classification problem, with
targets $\sign(\inner{M_i,v^{(t)}})$ and weights
$\abs{\inner{M_i,v^{(t)}}}$.  We approximate it as a weighted logistic
regression problem, and at each update iteration attempt to improve the
objective using a small number (e.g.~10) epochs of stochastic gradient
descent on the logistic loss.  For \LinLin, we also symmetrically
update rows of $W_d$.

When optimizing the model for some bit-length $k$, we
initialize to the optimal $k-1$-length model.  We initialize the new
bit either randomly, or by thresholding the rank-one projection of
$M$ (for unconstrained $U,V$) or the rank-one projection after
projecting the columns of $M$ (for \LinV) or both rows and columns of
$M$ (for \LinLin) to the column space of $X$.  We take the
initialization (random, or rank-one based) that yields a lower
objective value.

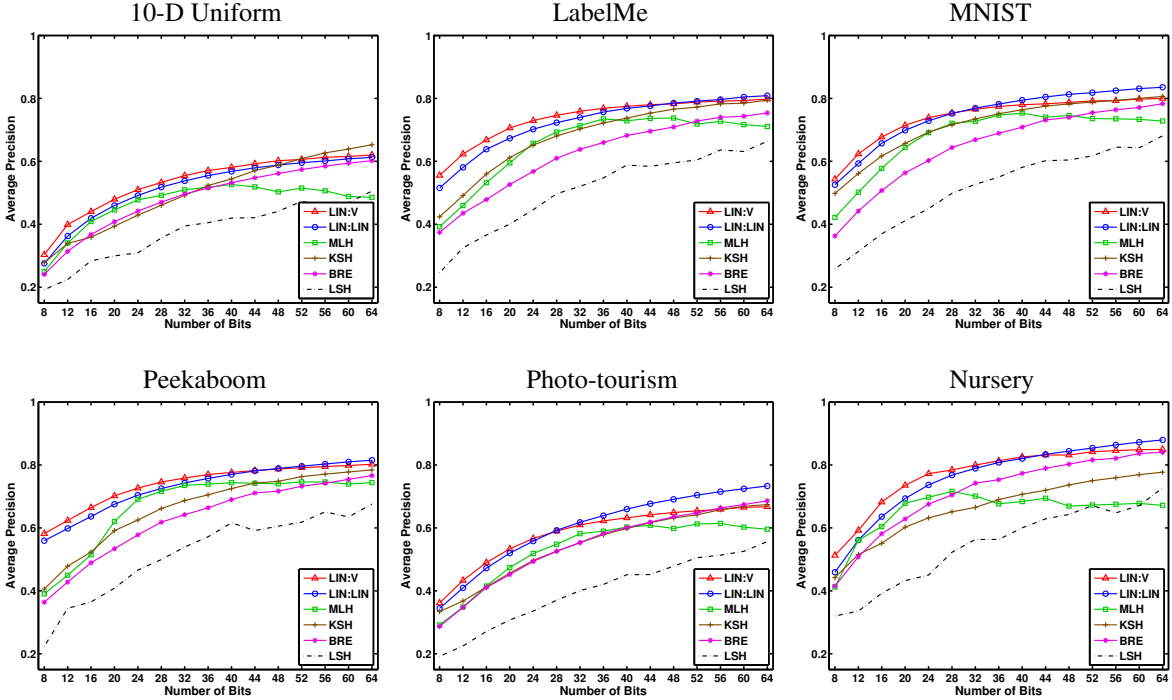
\begin{figure}[t!]
\vspace{0.2in}
\hbox{ \centering
\setlength{\epsfxsize}{2in}
\epsfbox{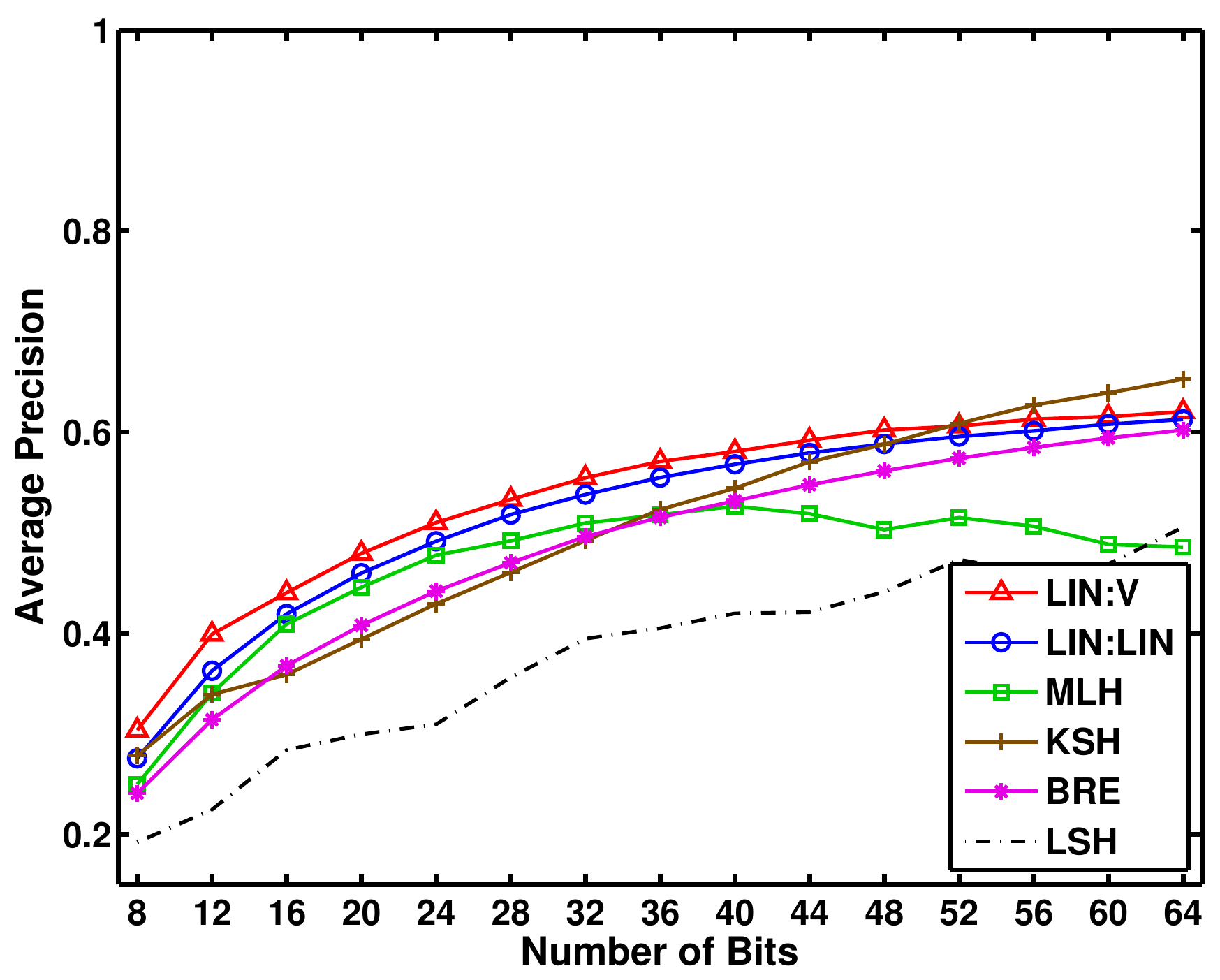}
\setlength{\epsfxsize}{2in}
\epsfbox{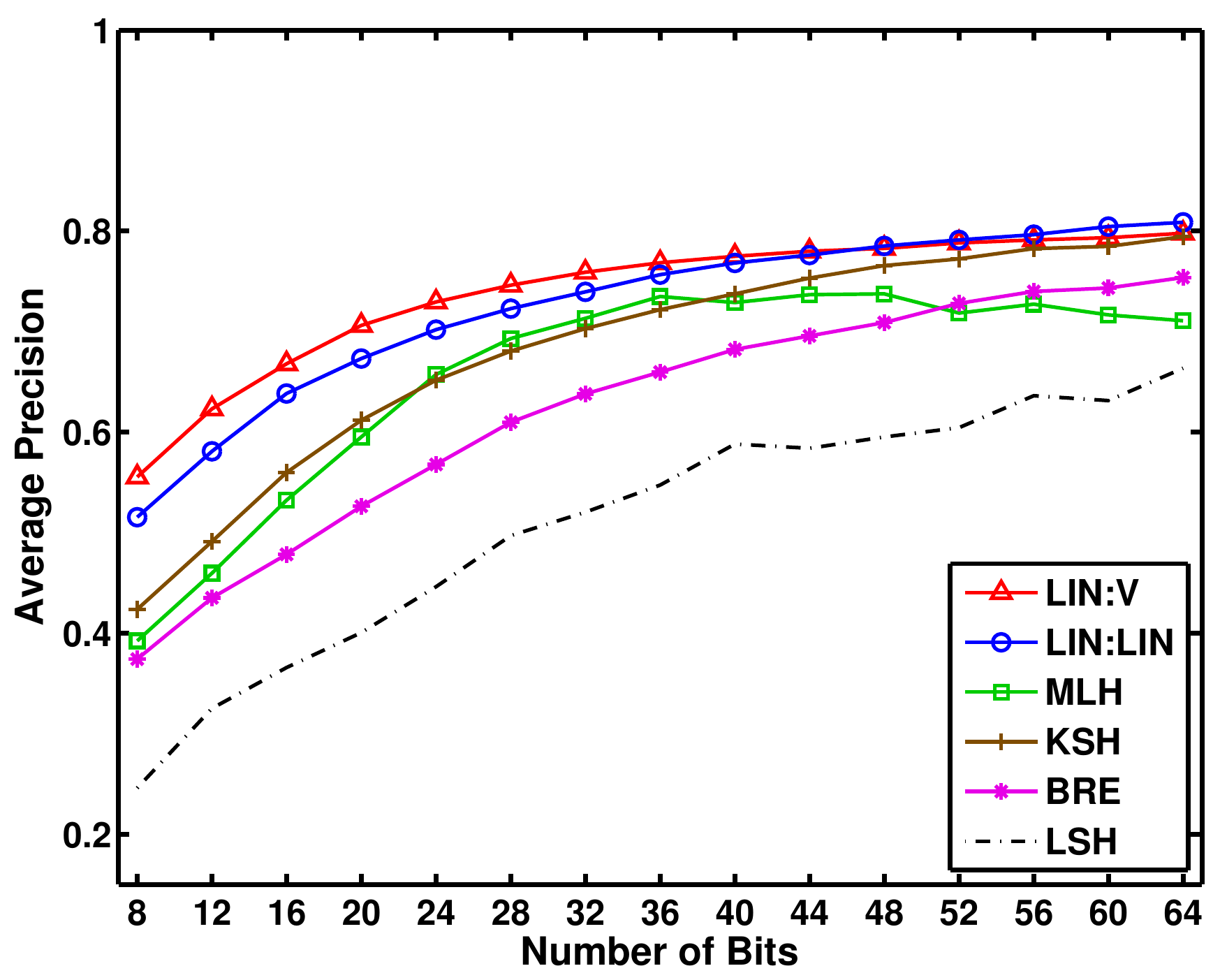}
\setlength{\epsfxsize}{2in}
\epsfbox{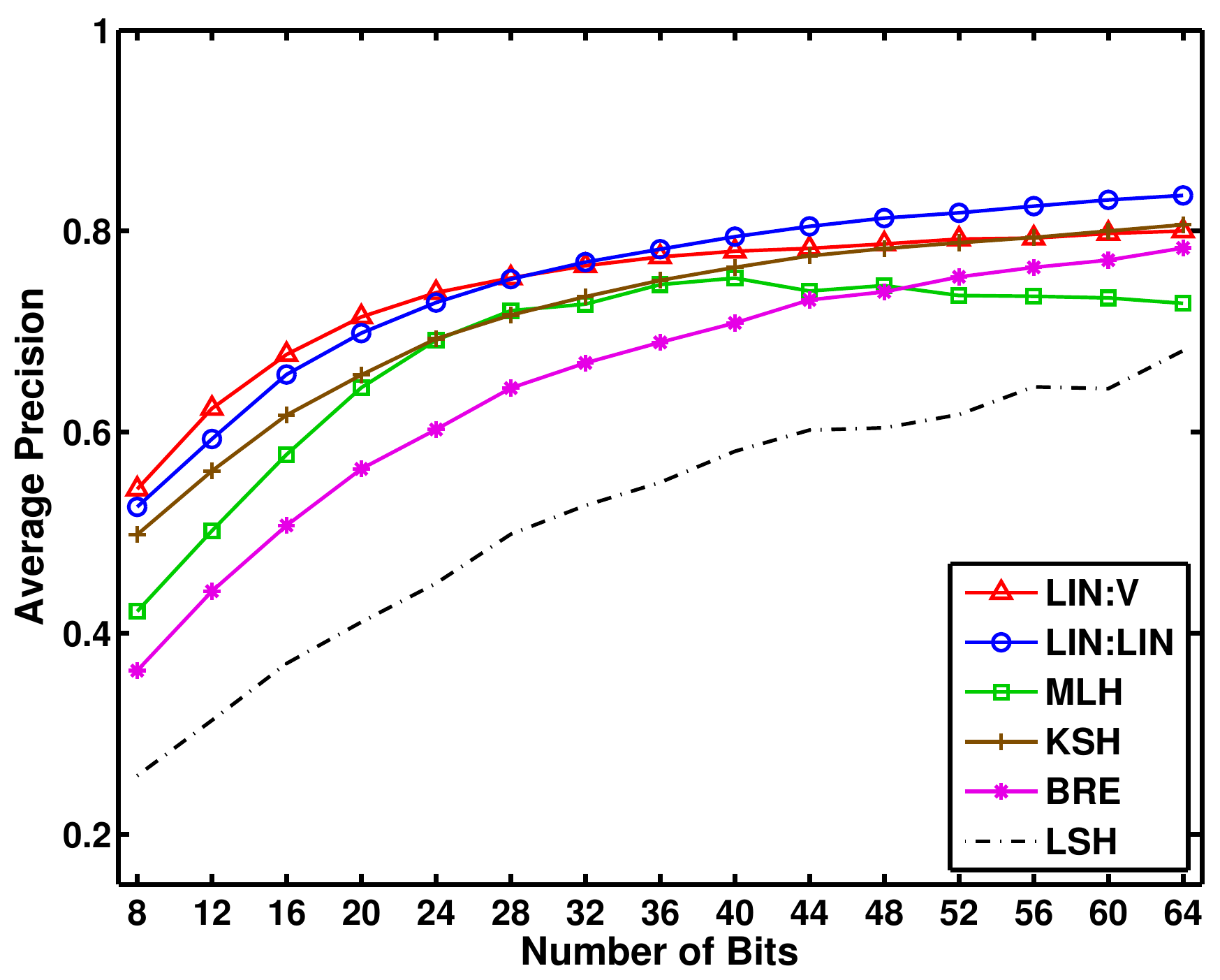}
\begin{picture}(0,0)(0,0)
\put(-400, +118){ 10-D Uniform}
\put(-240, +118){ LabelMe}
\put(-90, +118){ MNIST}
\end{picture}
}
\vspace{0.3in}
\hbox{ \centering 
\setlength{\epsfxsize}{2in}
\epsfbox{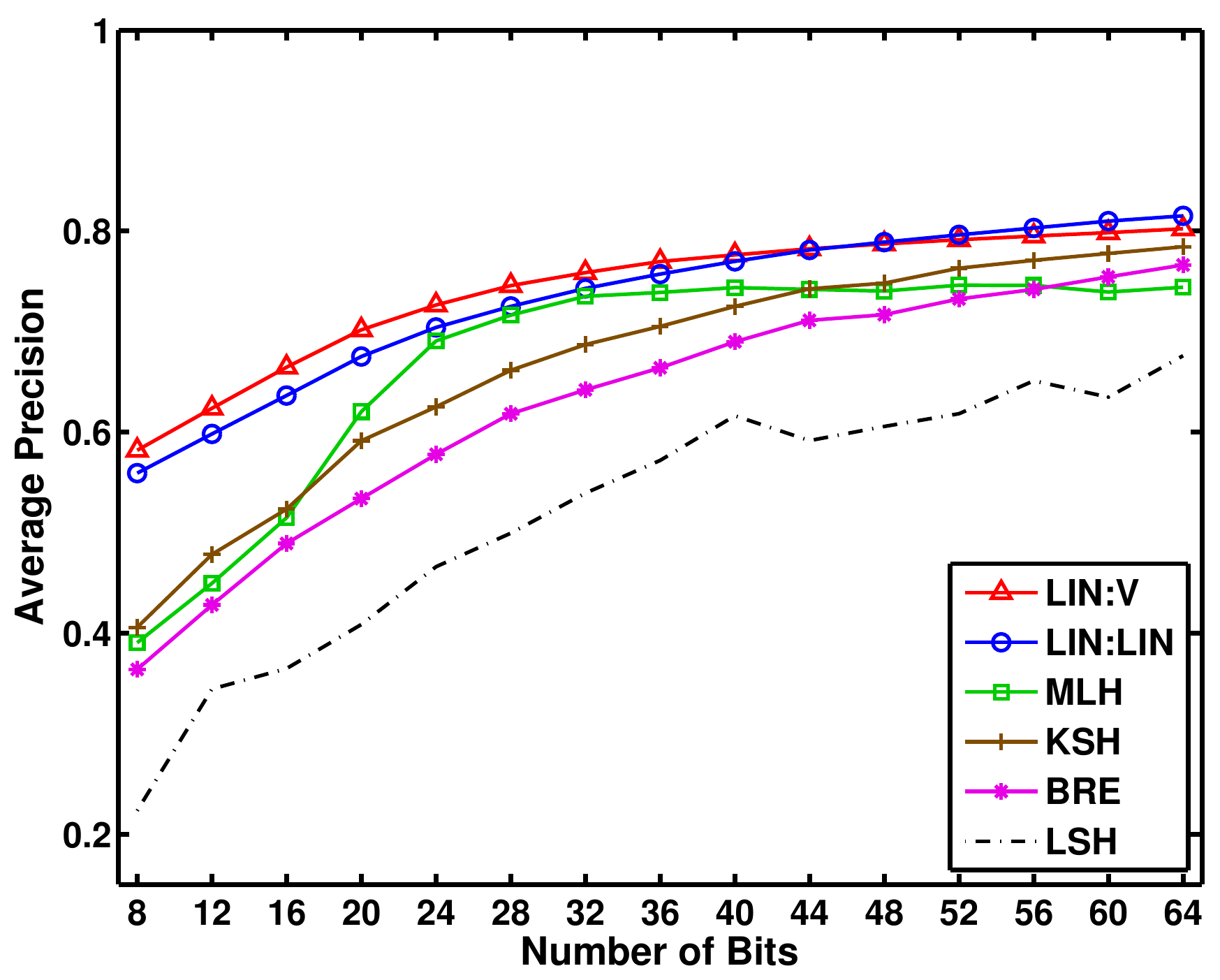}
\setlength{\epsfxsize}{2in} 
\epsfbox{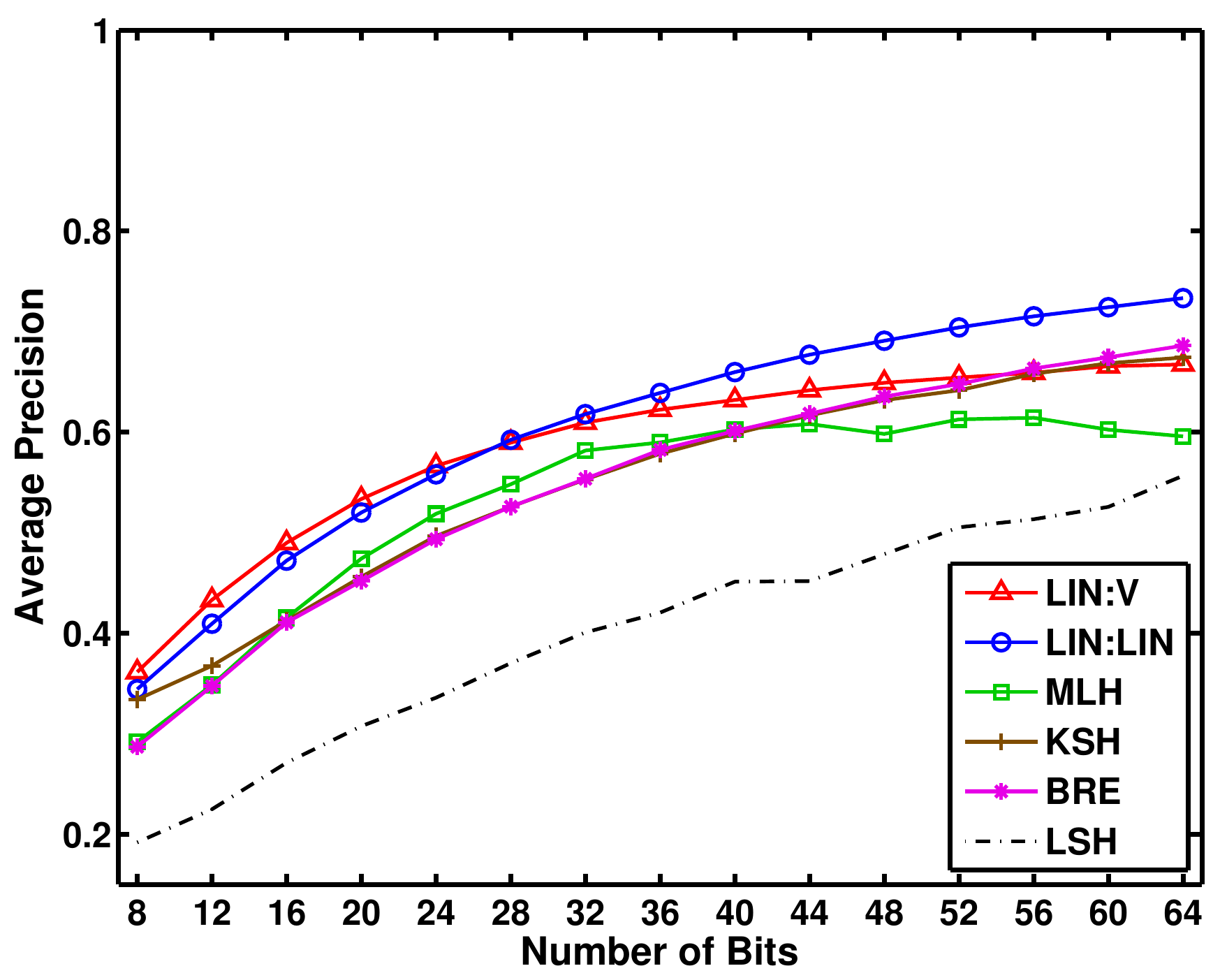}
\setlength{\epsfxsize}{2in}
\epsfbox{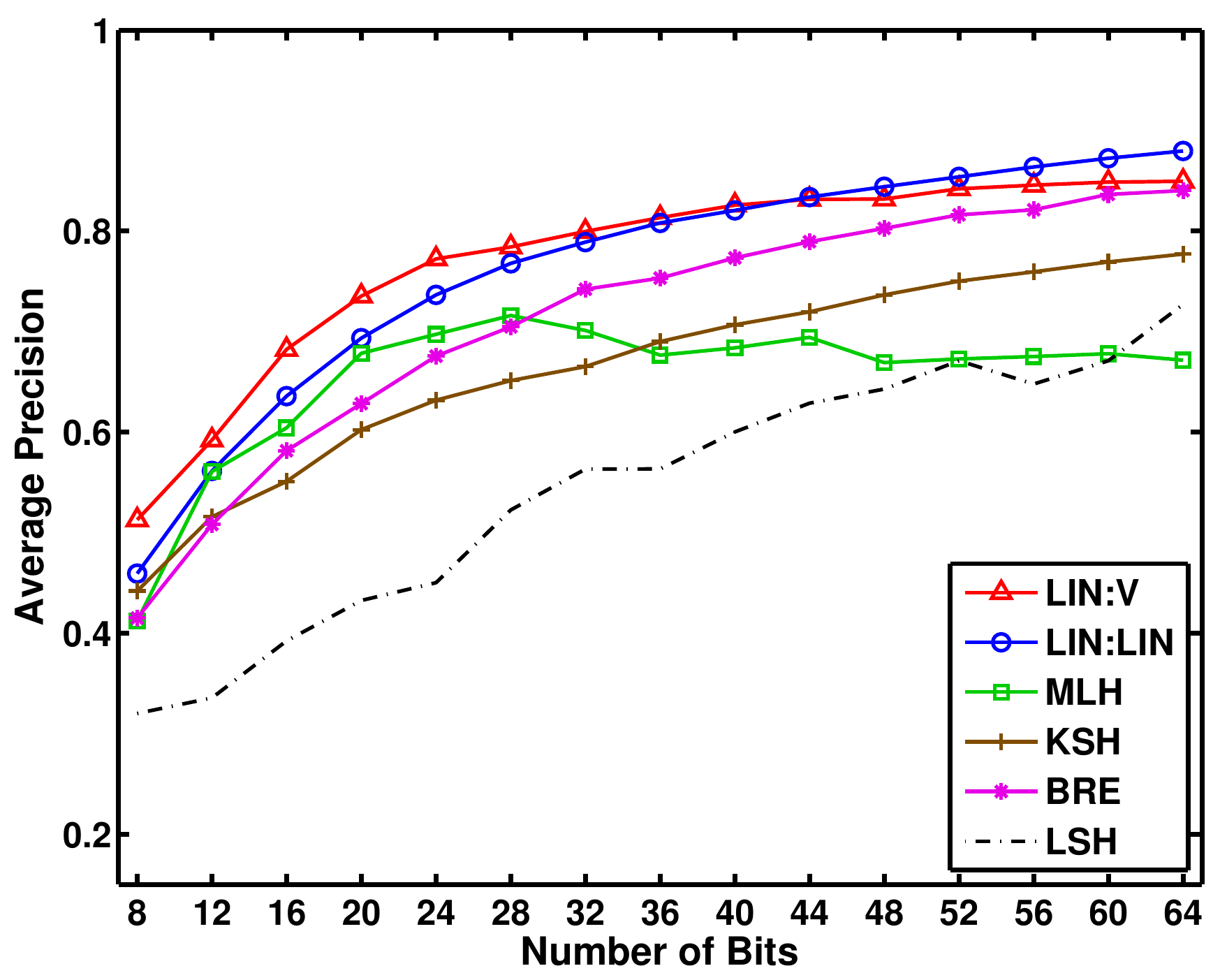}
\begin{picture}(0,0)(0,0)
\put(-395, +118){ Peekaboom}
\put(-250, +118){ Photo-tourism}
\put(-90, +118){ Nursery }
\end{picture}
}
\small\caption{\small Average Precision (AP) of points retrieved using Hamming distance as
a function of code length for six datasets. 
Five curves represent: LSH, BRE, KSH, MLH, and two variants of
our method: Asymmetric LIN-LIN and Asymmetric LIN-V. 
(Best viewed in color.)
}
\label{fig:AP}
\end{figure}

\section{Empirical Evaluation}\label{sec:emp}
In order to empirically evaluate the benefits of asymmetry in hashing,
we replicate the experiments of \cite{NorouziICML11}, which were in turn based
on  \cite{Kulis2009}, on six datasets using learned (symmetric)
linear threshold codes.
These datasets include:
LabelMe and Peekaboom are collections of images,
represented as 512D GIST features~\cite{TorralbaCVPR08}, Photo-tourism is a database of
image patches, represented as 128 SIFT features~\cite{SnavelySIGGRAPH06}, MNIST is a collection of
785D greyscale handwritten images, and Nursery contains 8D features.
Similar to \cite{NorouziICML11,Kulis2009}, we also constructed a synthetic 10D Uniform dataset,
containing uniformly sampled 4000 points for a 10D hypercube. We used 
1000 points for training and 3000 for  testing. 

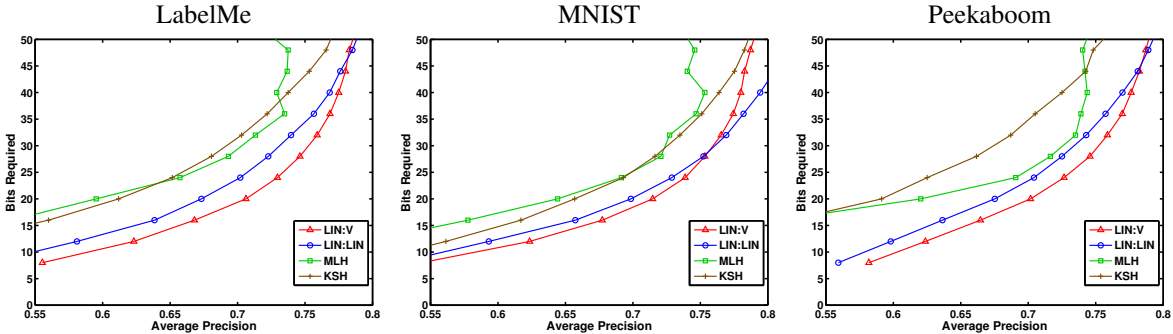
\begin{figure}[t!]
\vspace{0.2in}
\hbox{ \centering
\setlength{\epsfxsize}{2in}
\epsfbox{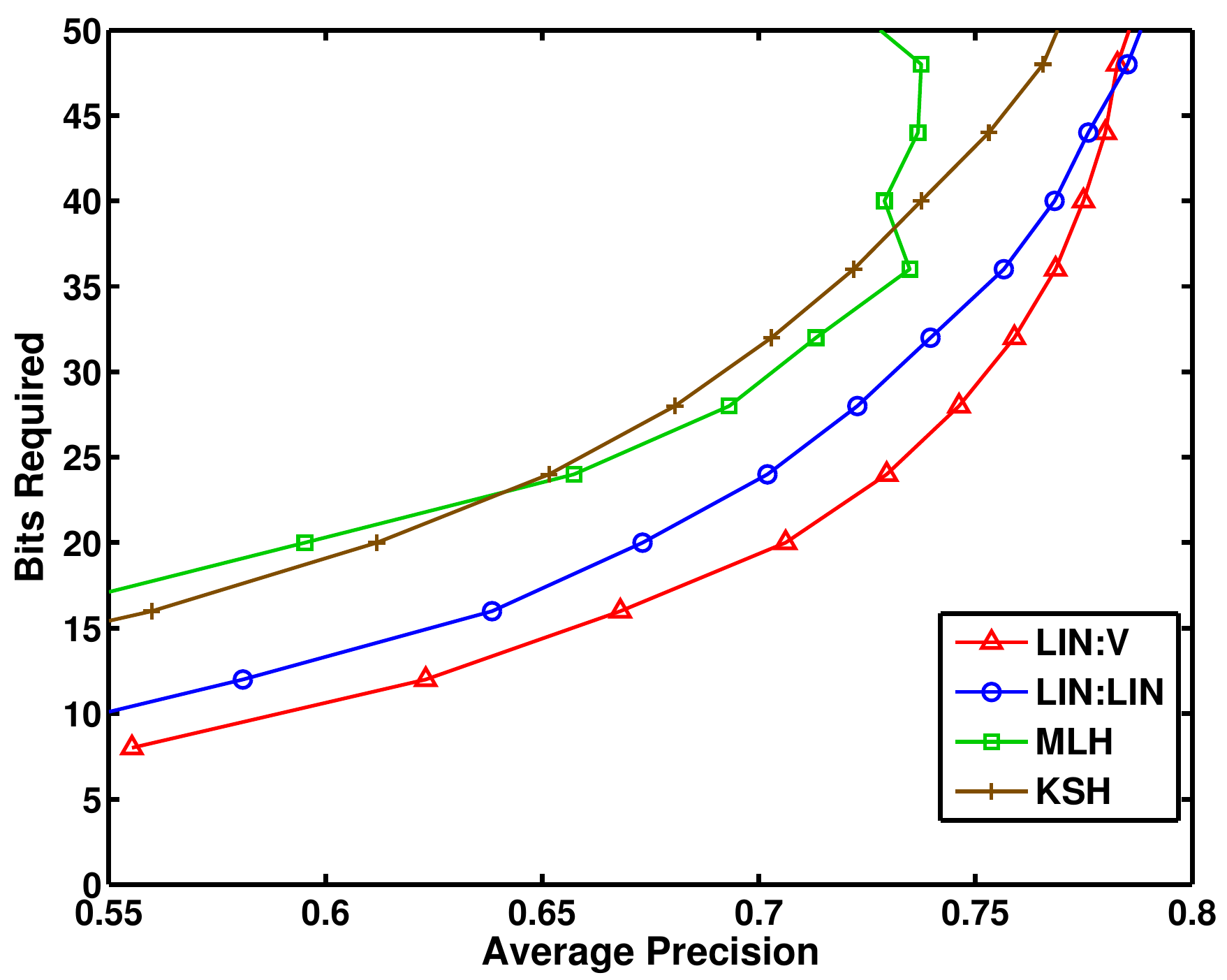}
\setlength{\epsfxsize}{2in}
\epsfbox{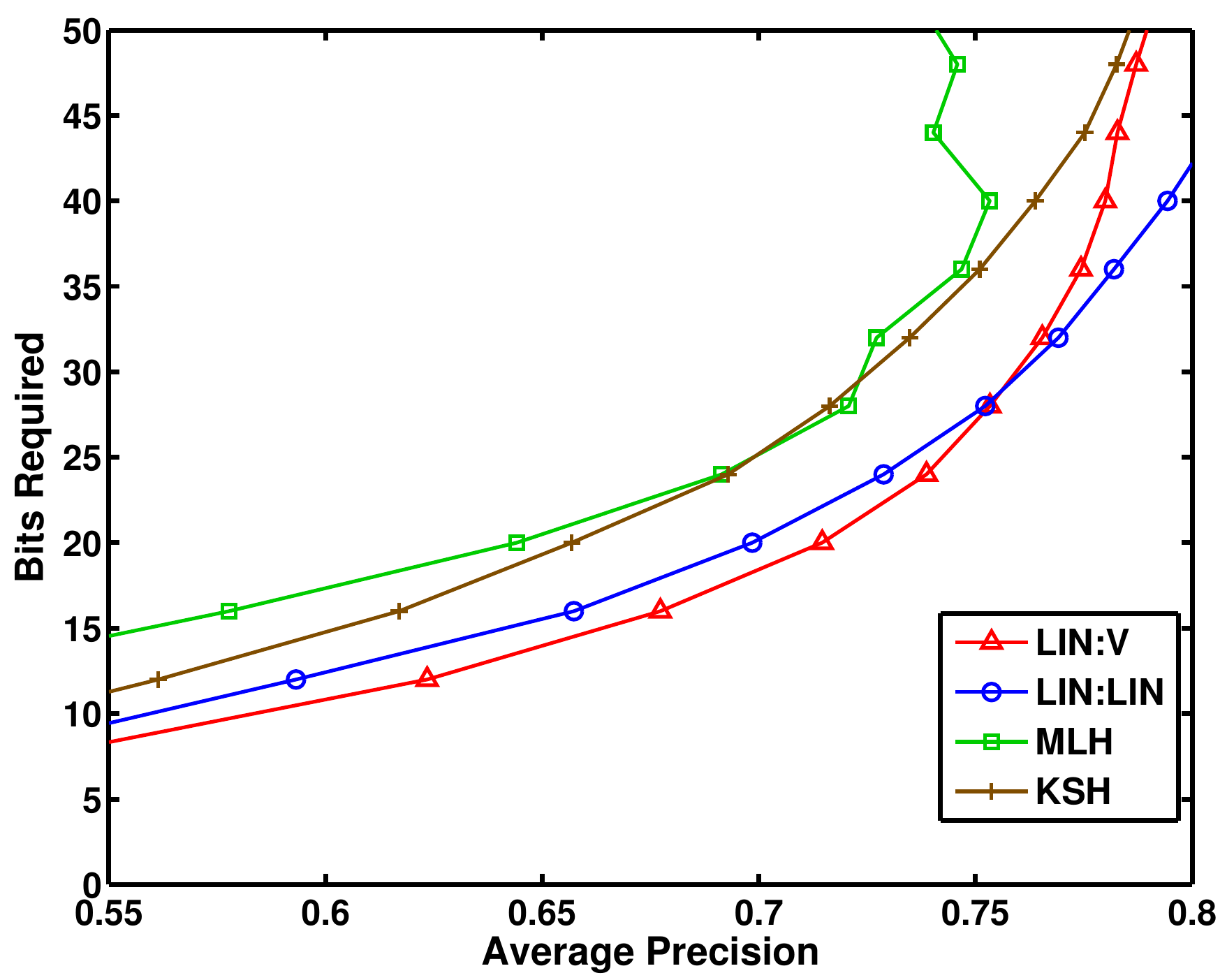}
\setlength{\epsfxsize}{2in}
\epsfbox{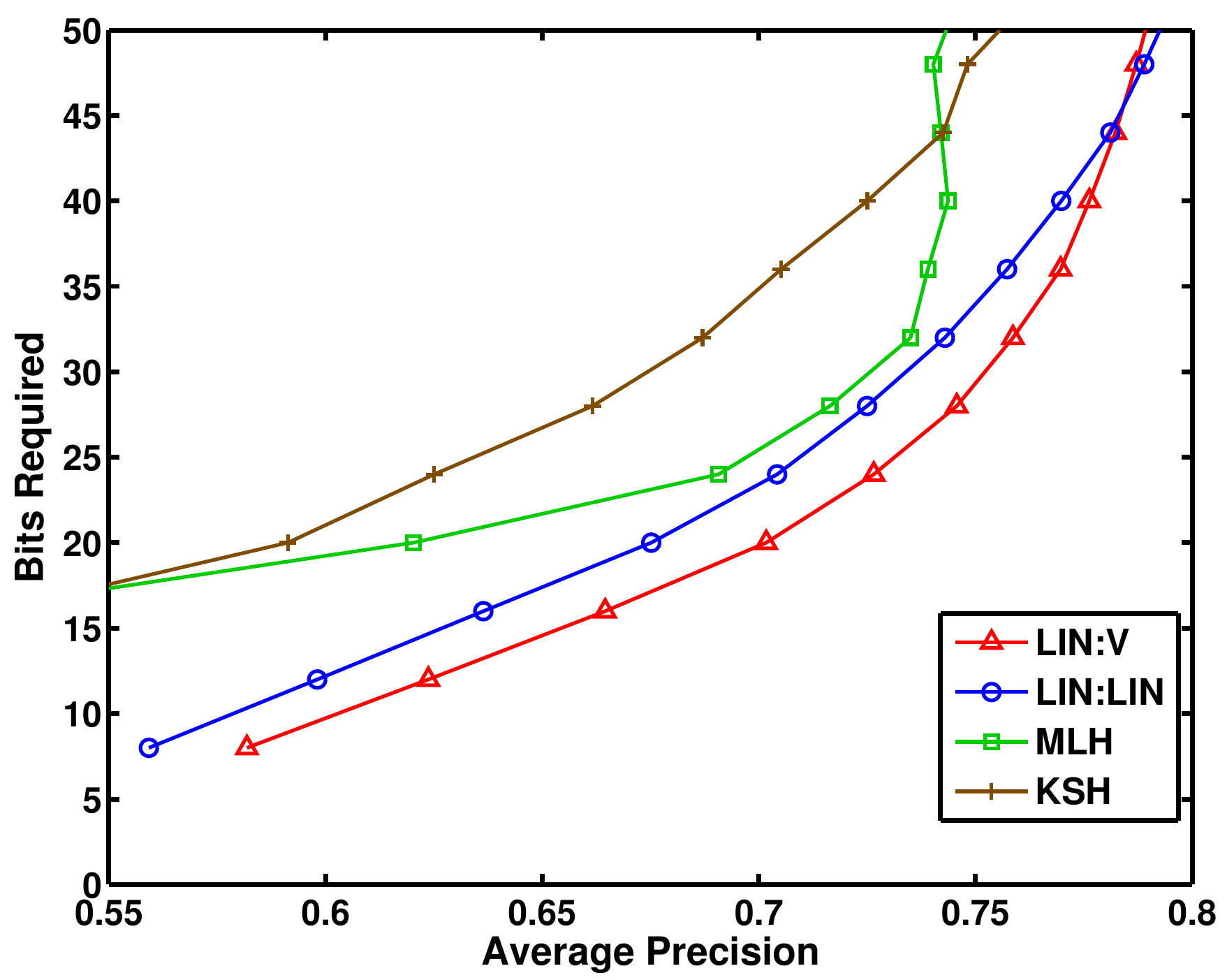}
\begin{picture}(0,0)(0,0)
\put(-390, +118){ LabelMe}
\put(-238, +118){ MNIST}
\put(-98, +118){ Peekaboom}
\end{picture}
}
\small\caption{\small Code length required as a function of Average Precision (AP) for three datasets.
}
\label{fig:bit}
\end{figure}

For each dataset,
we find the Euclidean distance at which each point has, on average, 50 neighbours.
This defines our ground-truth similarity in terms of neighbours and non-neighbours.
So for each dataset, we are given
a set of $n$ points $x_1,\ldots,x_n$, represented as vectors in
$\cX=\R^d$, and the binary similarities $S(x_i,x_j)$ between the
points, with +1 corresponding to $x_i$ and $x_j$ being neighbors and -1 otherwise.
Based on these $n$
training points, \cite{NorouziICML11} present a sophisticated optimization
approach for learning a thresholded linear hash function of the form
$f(x)=\sign(Wx)$, where $W\in\R^{k\times d}$.  This hash function is
then applied and $f(x_1),\ldots,f(x_n)$ are stored in the database.
\cite{NorouziICML11} evaluate the quality of the hash by considering
an independent set of {\em test points} and comparing $S(x,x_i)$ to
$\sign(\inner{f(x),f(x_i)}-\theta)$ on the test points $x$ and the
database objects (i.e.~training points) $x_i$.  

In our experiments, we followed the same protocol, but with the two
asymmetric variations $\LinLin$ and $\LinV$, using the
optimization method discussed in Sec.~\ref{Sec:Optimization}. 
In order to obtain different balances between precision and recall, we
should vary $\beta$ in \eqref{eq:aproxUV}, obtaining different codes
for each value of $\beta$.  However, as in the experiments of
\cite{NorouziICML11}, we actually learn a code (i.e.~mappings
$f(\cdot)$ and $g(\cdot)$, or a mapping $f(\cdot)$ and matrix $V$)
using a fixed value of $\beta=0.7$, and then only vary the threshold
$\theta$ to obtain the precision-recall curve.

In all of our experiments, in addition to Minimal Loss Hashing (MLH), we also 
compare our approach to three other widely used methods: Kernel-Based Supervised Hashing (KSH) of \cite{LiuCVPR12}, Binary Reconstructive Embedding (BRE) of \cite{Kulis2009}, and
Locality-Sensitive Hashing (LSH) of~\cite{DatarIIM04}. \footnote{We used the BRE, KSH and MLH implementations available from the original authors.  For each method, we followed the instructions provided by the authors. More specifically, we set the number of points for each hash
function in BRE to 50 and the number of anchors in KSH to 300 (the default values). For MLH, we learn the threshold and shrinkage parameters by cross-validation and other parameters are initialized to the suggested values in the package.}

In our first set of experiments, we test performance of the
asymmetric hash codes as a function of the bit length.
Figure~\ref{fig:AP} displays Average Precision (AP) of data points retrieved using Hamming distance as
a function of code length.
These results are similar to ones reported by \cite{NorouziICML11}, where
MLH yields higher precision compared to BRE and LSH. Observe that for all six datasets
both variants of our method, asymmetric \LinLin and asymmetric \LinV,
consistently outperform all other methods for different binary code length.
The gap is particularly large for short codes.  
For example, for the LabelMe dataset, MLH and KSH with 
16 bits achieve AP of 0.52 and 0.54 respectively, whereas \LinV
already achieves AP of 0.54 with only 8 bits. Figure~\ref{fig:bit} shows similar performance gains 
appear for a number of other datasets. We also note across all datasets 
\LinV improves upon \LinLin for short-sized codes. 
These results clearly show that an asymmetric binary hash can be much more compact 
than a symmetric hash.  

Next, we show, in Figure~\ref{fig:PR}, the full Precision-Recall curves
for two datasets, LabelMe and MNIST, and for two specific code
lengths: 16 and 64 bits.  The performance of \LinLin and \LinV is
almost uniformly superior to that of MLH, KSH and BRE methods.  We observed
similar behavior also for the four other datasets across various
different code lengths.

\begin{figure}[t!]
\vspace{0.3in}
\hbox{ \centering 
\setlength{\epsfxsize}{1.5in}
\epsfbox{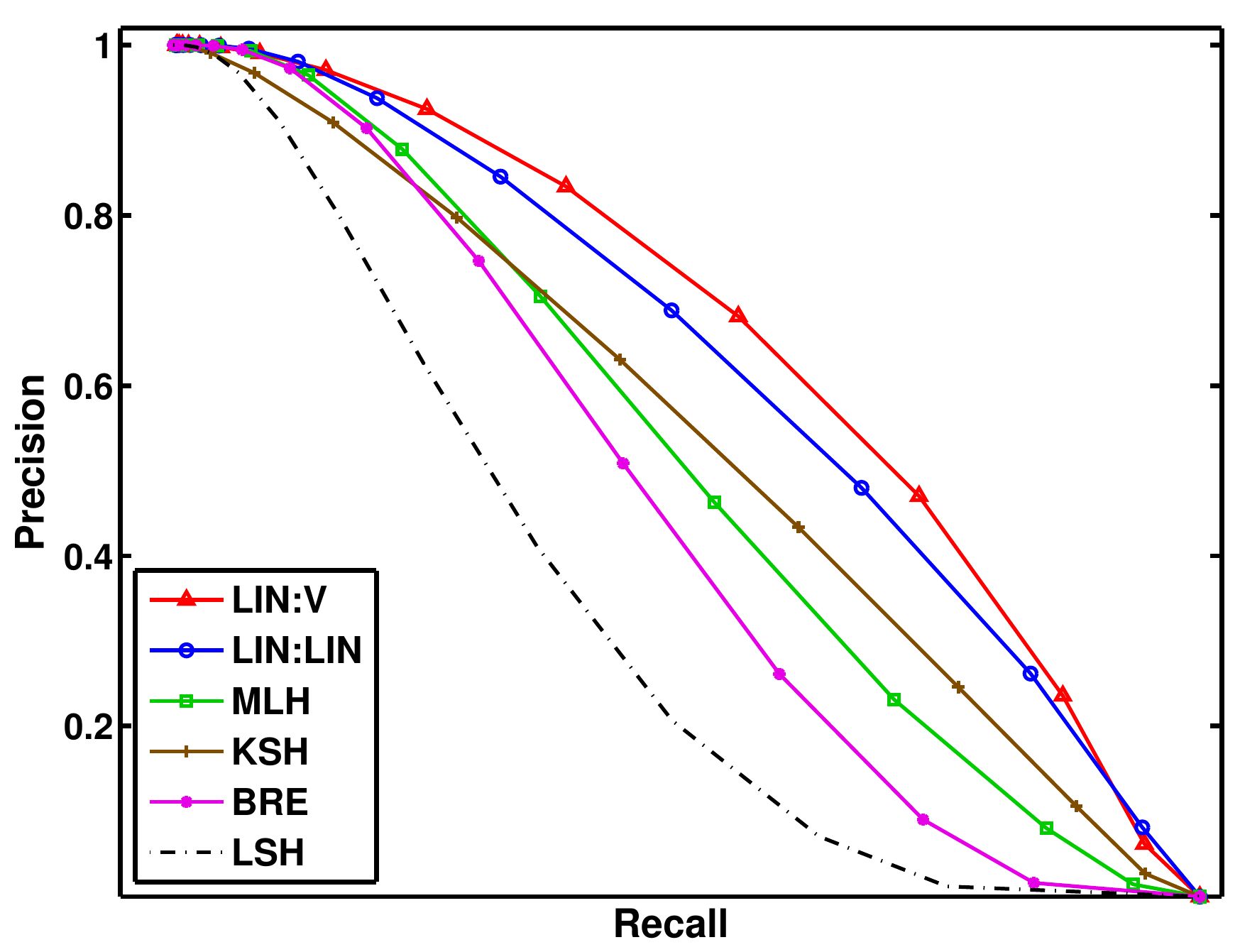}
\hspace{-0.1in}
\setlength{\epsfxsize}{1.5in}
\epsfbox{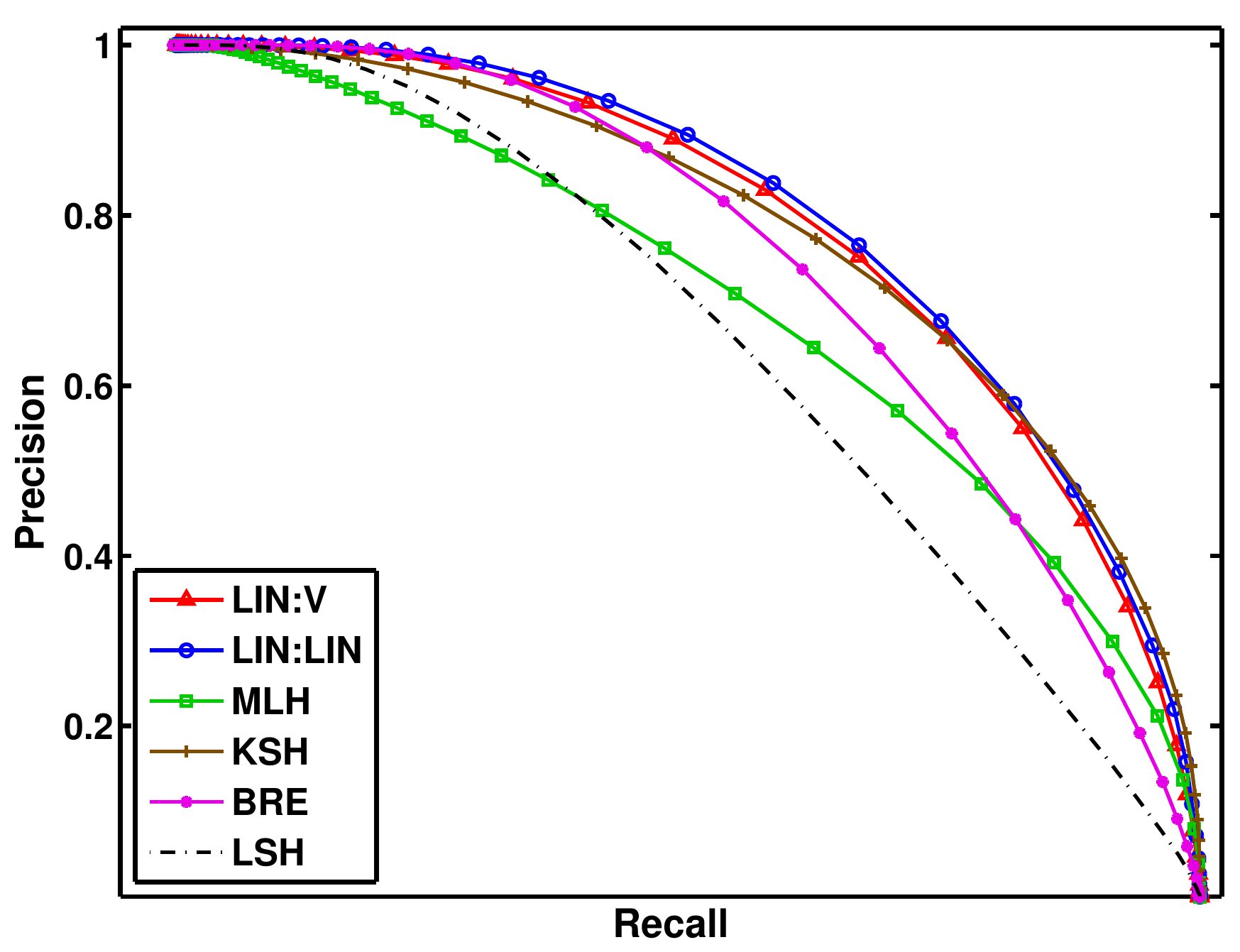}
\hspace{0.1in}
\setlength{\epsfxsize}{1.5in}
\epsfbox{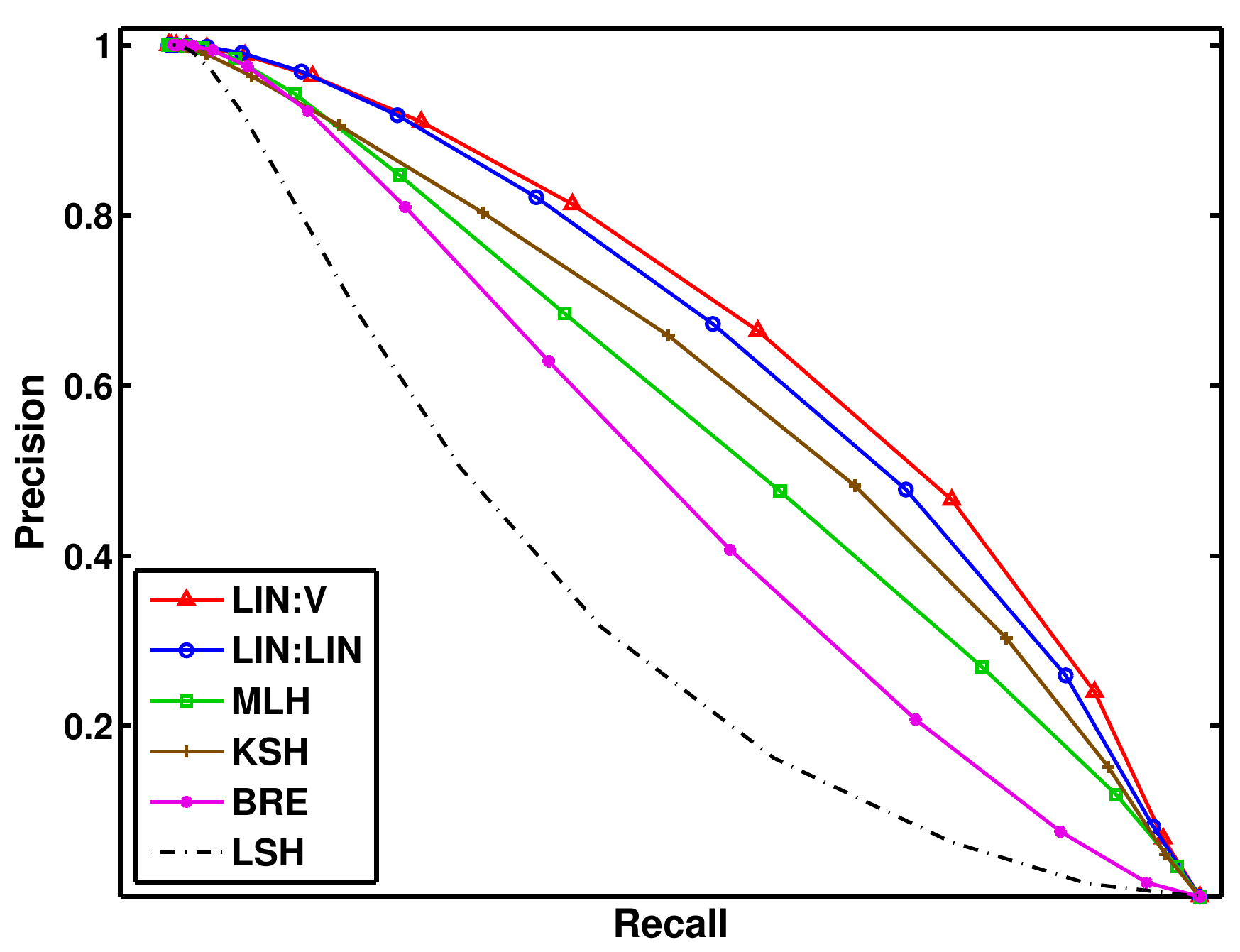}
\hspace{-0.1in}
\setlength{\epsfxsize}{1.5in}
\epsfbox{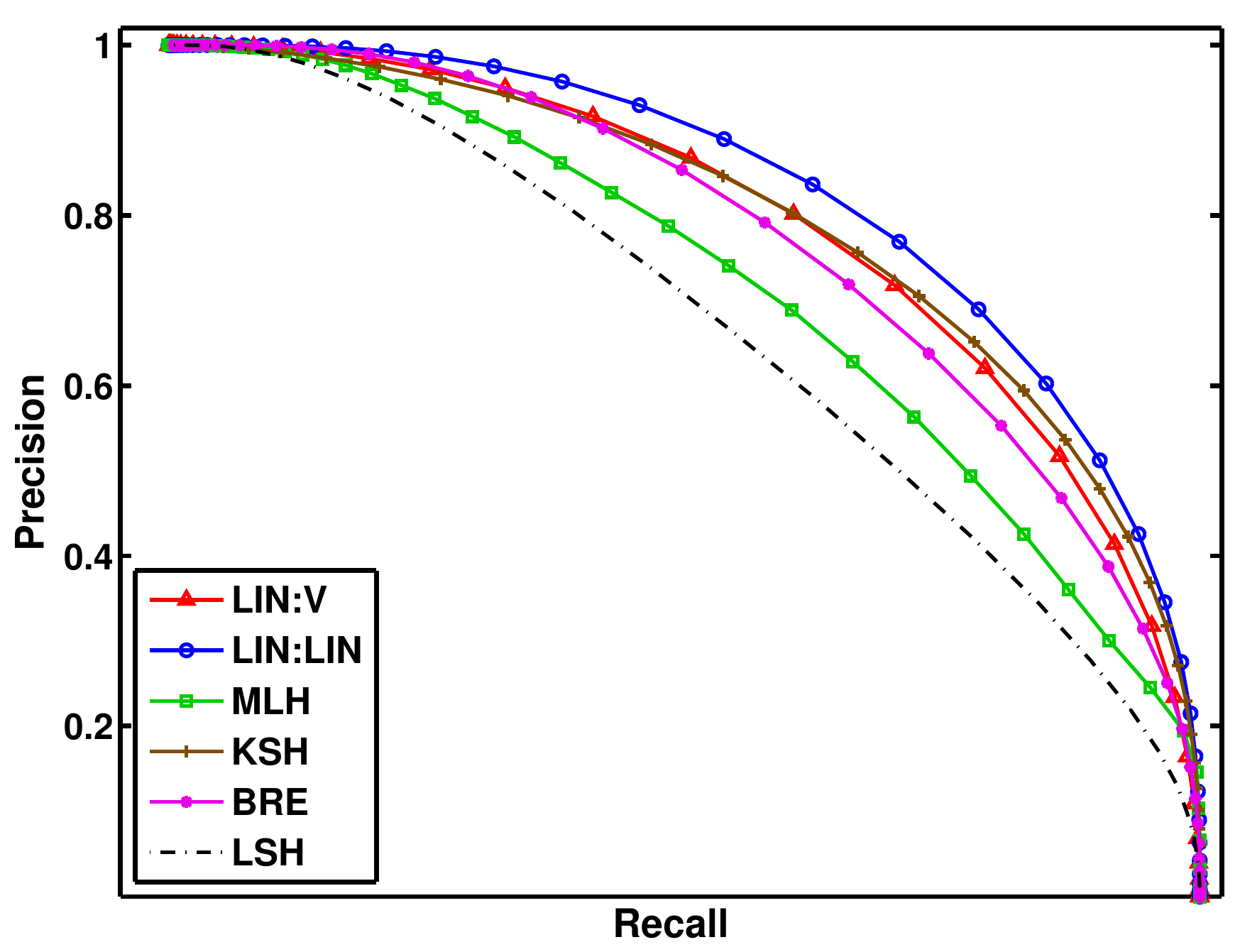}
\begin{picture}(0,0)(0,0)
\put(-410, +85){ 16 bits}
\put(-300, +85){ 64 bits}
\put(-180, +85){ 16 bits}
\put(-70, +85){ 64bits}
\put(-360, +100){ LabelMe}
\put(-125, +100){ MNIST}
\end{picture}
}
\small\caption{\small Precision-Recall curves for LabelMe and MNIST datasets using 16 and 64 binary codes.
(Best viewed in color.)
}
\label{fig:PR}
\end{figure}
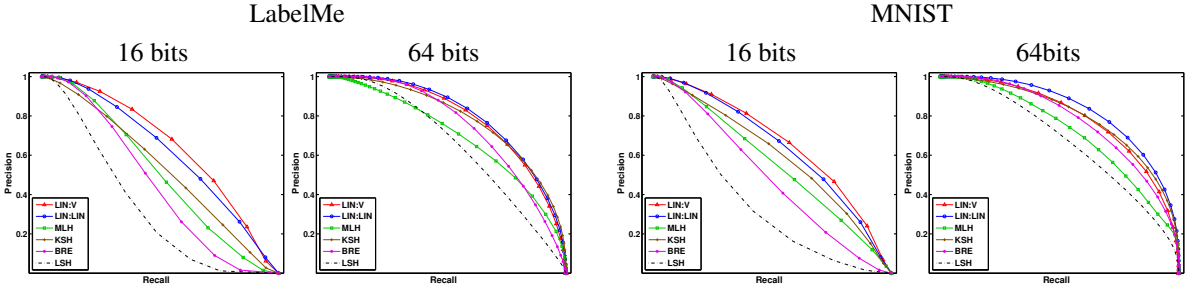


\begin{figure}[t!]
\vspace{0.1in}
\hbox{ \centering 
\setlength{\epsfxsize}{3.1in}

\epsfbox{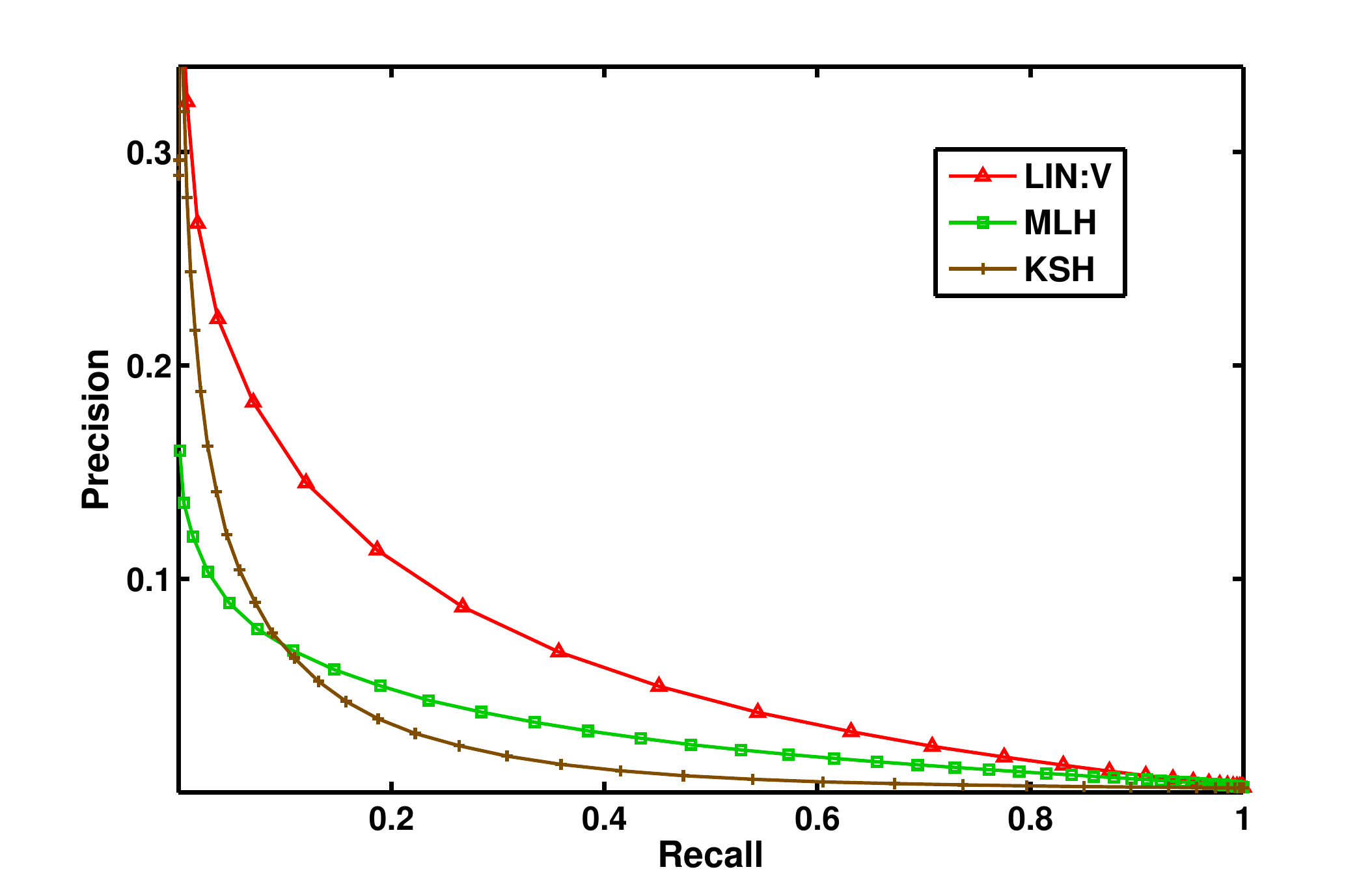}
\setlength{\epsfxsize}{3.1in}
\epsfbox{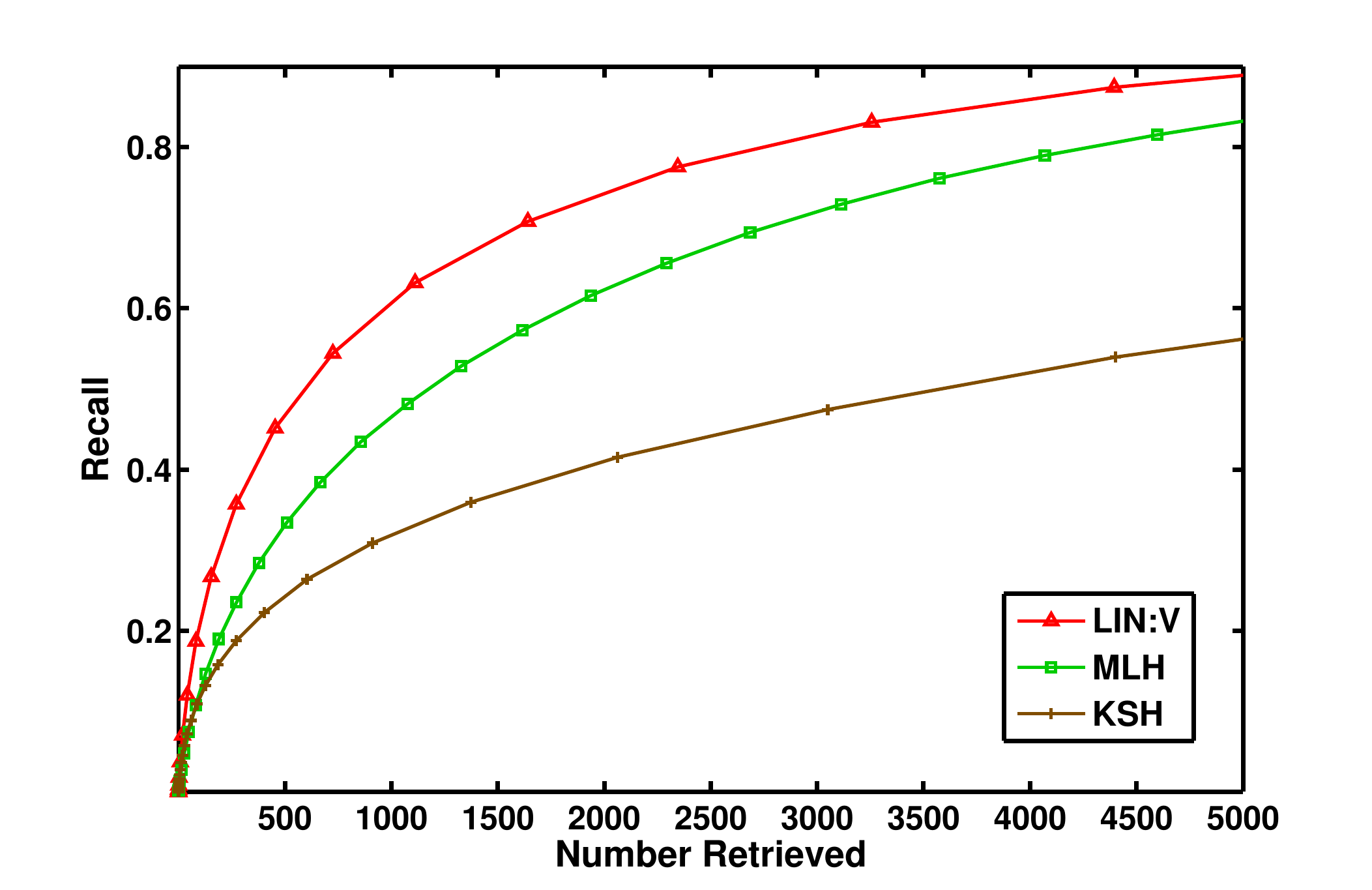}

}
\small\caption{\small {\bf Left:} Precision-Recall curves for the Semantic 22K LabelMe dataset
{\bf Right:} Percentage of 50 ground-truth neighbours as a function of retrieved images.
 (Best viewed in color.)} 
\label{fig:22K}
\end{figure}


Results on previous 6 datasets show that asymmetric binary codes can
significantly outperform other state-of-the-art methods on relatively
small scale datasets.  We now consider a much larger
LabelMe dataset~\cite{TorralbaCVPR08}, called {\it Semantic 22K
  LabelMe}.  It contains 20,019 training images and 2,000 test images,
where each image is represented by a 512D GIST descriptor.  The
dataset also provides a semantic similarity $S(x,x')$ between two
images based on semantic content (object labels overlap in two
images). As argued by ~\cite{NorouziICML11}, hash functions learned
using semantic labels should be more useful for content-based image
retrieval compared to Euclidean distances.  Figure~\ref{fig:22K} shows
that \LinV with 64 bits substantially outperforms MLH and KSH with 64 bits.

\section{Summary} 

The main point we would like to make is that when considering binary
hashes in order to approximate similarity, even if the similarity
measure is entirely symmetric and ``well behaved'', much power can be
gained by considering asymmetric codes.  We substantiate this claim by
both a theoretical analysis of the possible power of asymmetric codes,
and by showing, in a fairly direct experimental replication, that
asymmetric codes outperform state-of-the-art results obtained for
symmetric codes.  The optimization approach we use is very crude.
However, even using this crude approach, we could find
asymmetric codes that outperformed well-optimized symmetric codes.  It
should certainly be possible to develop much better, and more
well-founded, training and optimization procedures.

Although we demonstrated our results in a specific setting using
linear threshold codes, we believe the power of asymmetry is far more
widely applicable in binary hashing, and view the experiments here as
merely a demonstration of this power.  Using asymmetric codes instead
of symmetric codes could be much more powerful, and allow for shorter
and more accurate codes, and is usually straightforward and does not
require any additional computational, communication or significant
additional memory resources when using the code.  We would therefore
encourage the use of such asymmetric codes (with two distinct hash
mappings) wherever binary hashing is used to approximate similarity.
\subsubsection*{Acknowledgments}
This research was partially supported by NSF CAREER award CCF-1150062 and NSF grant IIS-1302662.
\bibliographystyle{plain}
\bibliography{asym}
\end{document}